\documentclass[twoside,11pt]{article}

\usepackage{jmlr2e}
\usepackage{mathrsfs,amssymb,amsmath,booktabs,array,xcolor,url,cite,color,soul,multirow,multicol}
\usepackage{mathbbold,bm,bbm}
\usepackage{slashbox}

\ShortHeadings{Knowledge Integrated Classifier Design Based on Utility Optimization}{Shaohan Chen and Chuanhou Gao}
\firstpageno{1}

\begin{document}

\title{Knowledge Integrated Classifier Design Based on Utility Optimization}

\author{\name Shaohan Chen \email shaohan\_chen@zju.edu.cn \\
       \addr School of Mathematical Sciences\\
       Zhejiang University\\
       Hangzhou 310027, China
       \AND
       \name Chuanhou Gao \email gaochou@zju.edu.cn \\
       \addr School of Mathematical Sciences\\
       Zhejiang University\\
       Hangzhou 310027, China}

\editor{}

\maketitle

\begin{abstract}
This paper proposes a systematic framework to design a classification model that yields a classifier which optimizes a utility function based on prior knowledge. 
Specifically, as the data size grows, we prove that the produced classifier asymptotically converges to the optimal classifier, an extended version of the Bayes rule, which maximizes the utility function. Therefore, we provide a meaningful theoretical interpretation for modeling with the knowledge incorporated.
Our knowledge incorporation method allows domain experts to guide the classifier towards correctly classifying data that they think to be more significant. 
\end{abstract}


\section{Introduction}

The past few decades are giant leap for machine learning (ML) researches, as the development of supervised learning, unsupervised learning, and reinforcement learning,  etc. 
Especially, the rapid progress of ML in the recent few years dues to advances in deep learning together with the availability of GPUs. 
These techniques have rapidly helped us to make decisions or find optimal control policies for complex tasks, such as iron-making processes \citep{gao2014rule}, diagnosing diseases \citep{esteva2017dermatologist}, autonomous cars \citep{bojarski2016end}, etc. 

However, the lack of theoretical interpretability for many variants of deep learning has been an obstacle to applying them with confidence. Thus it is difficult to adjust or design these model with deep structures to satisfy different demands in different fields. Incorporating human experience into machine learning algorithms can greatly help them to make decisions without learning from scratch or making foolish decisions. Therefore, our motivation is to design a machine learning algorithm by integrating prior information into classification models to maximize the utility(the definition of utility function will be given below) of decision-making in the framework of kernel methods, which has the perfect theoretical basis. 

This paper will focus on binary classification problems. Classical methods, such as standard support vector machines (SVMs) in both offline \citep{cortes1995support} and online \citep{kivinen2004online} settings, which are designed for classifying problems with an equal cost for both classes and they have been very effective on several application problems. 
However, it is also limiting the applications of these models, for example, when the data is highly unbalanced. Lin et al. \citep{lin2002support} modified the standard SVMs for classification when the costs of misclassifying positive and negative labeled data are unequal to minimize the expected cost. 
There are many other classifiers design methods to optimize some special purposes rather than the commonly used performance measure: accuracy. For example, 
Gao et al. \citep{gao2013one} introduced a one-pass optimal AUC (Area under the ROC curve) method where the data is highly unbalanced.
Xu et al. \citep{xu2014classifier} presented an algorithm to  balance the performance with test-time cost efficiently. 
 Lin et al. \citep{lin2011design} designed a privacy-preserving SVM classifier to protect the privacy of data. 
 Wu et al. \citep{wu2004incorporating} proposed a weighted SVM model to incorporate some human knowledge represented by the weights of every training data point, etc.
 Meanwhile, there has been some literature providing the theoretical guarantees for convergence or the asymptotic performances with these inventive classifier design methods.
  
 It has been strictly proved the convergence and asymptotic optimal properties of standard SVMs in both offline \citep{chen2004support}, \citep{wu2006analysis} and online \citep{ying2006online} settings. Concretely, these works indicated that the estimator of SVM is asymptotically equivalent to the best classifier, Bayes rule.
  Lin et al. \citep{lin2002support} have shown the explanations of the degenerated non-standard SVM where the costs for misclassifying different classes are unequal.
 Gao et al. \citep{gao2013one} also verified the theoretical effectiveness of the proposed one-pass optimal AUC model by analyzing its convergence.
 
Given the above analysis, we hope to deeply understand the correlations between the classification model and its yield classifier's (asymptotic) performance. Then conversely, we can arbitrarily set up the desired performance of classifier by designing a unique classification model to achieve. 
 
The incorporation of prior information into machine learning algorithms is the key element that allows increasing the performance in many applications \citep{niyogi1998incorporating}, \citep{lauer2008incorporating}, if the number of training data is limited. There is much literature about investigating the methods to incorporate prior knowledge into black-box models. For example, Mangasarian et al. \citep{mangasarian2004knowledge}, \citep{mangasarian2007nonlinear},
\citep{mangasarian2008nonlinear} converted the prior information in the form of logical implications into constraints to which classification model should respect. However, there is no sufficient research on theoretically validating the efficiency of integrating prior knowledge into machine learning models. 

Thus, we wish to integrate prior knowledge into the black-box model optimally by maximizing a utility function. In particular, our goal is to design a classification learning framework that yields a classifier which maximizes a pre-defined utility function with which prior knowledge is combined. As one of the contributions in this paper, we will show later that the proposed knowledge-based classifier is asymptotically equivalent to the optimal classifier, an extended version of the Bayes rule, which maximizes the utility function.

To illustrate the significance of the presented utility function which the built classifier hopes to optimize, we will take the problem of diagnosing diseases as an example. Machine learning systems can make untrusted decisions or wrong predictions, i.e., missing the disease diagnosis of a patient whose physical examination items are too far below or above the safe ranges can lead to serious consequences. In order to design a classifier to avoid happening of the missed diagnosis, we try to combine the black-box modeling techniques with doctors' experience to minimize the risk of missing diagnosis. In other words, for patients receiving the safe treatments from the produced classifier, we need to consider the doctors' requirements in our learning framework to reduce the risk of happening medical accidents as low as possible. 

Generally, for machine learning systems to be used safely, it is critical to satisfying auxiliary criteria and not only accuracy, i.e., the domain experts' requirements(knowledge). Specifically, we wish to build a classifier to optimally trade off different kinds of risks according to domain experts in applications. Therefore, our goal is to design a classification model which produces a classifier that can understand \textbf{risk preferences} in different application areas and minimize them. 
In this paper, we try to quantify different types of risks that may lead to severe consequences in applications and refer to the expectation of ensemble negative risks as the \textbf{utility function} formally. More details can be seen in section $3$.




The rest of the paper is organized as follows.  
In section $2$, we introduce the basic concepts of the standard SVMs and the corresponding Bayes-risk consistency property \citep{chen2004support}, \citep{wu2006analysis}. Also, we find that many varieties of SVMs are built for incorporating prior knowledge to satisfy different purposes.  Section $3$ is the critical part of this paper. Firstly, we give the definition of prior knowledge used in this paper, based on which we present the concept of the utility function as a new learning target in our study framework. Secondly, a knowledge-based classification model is built for optimizing the proposed utility function. Finally, we prove the yielded classifier can asymptotically converge to the optimal one which maximizes the utility function. In section $4$, we illustrate the significance of our knowledge incorporation learning scheme by comparing to other knowledge incorporation methods in more profound levels. 
Section $5$ concludes this paper and presents the future work.

\section{Background}
In this part, we firstly list some notations that will be used in this paper. Secondly, we introduce the basic concepts of the standard SVMs and the corresponding Bayes-risk consistency property \citep{chen2004support}, \citep{wu2006analysis}.  Finally, we try to discuss some other varieties of SVMs.
\subsection{Notations and Preliminaries}
We consider the following setup. Let $(X,d)$ be a compact metric space in $\mathbb{R}^n$ and $Y=\left\{1,-1\right\}$.
We are given a training data set of $m$ examples  $\textbf{z} = \left( z _ { i } \right) _ { i = 1} ^ { m }=\left( x _ { i } ,y _ { i } \right) _ { i = 1} ^ {m}$, which are i.i.d sampled according to a Borel probability measure $\rho$ on the space $Z : = X \times Y$ and $( \mathcal { X } ,\mathcal { Y } )$ be the corresponding random variable. 
In this paper, we focus on the binary classification problems.

A binary classifier $f : X \rightarrow \{ 1,- 1\}$ is a function that maps input pattern from $X$ to $Y$ which divides the input space $X$ into two classes. 
Generally, A classification algorithm is a map from the set of samples to a set of classifiers $\mathcal{H}$:
\begin{eqnarray}
\mathcal { Q }: U _ { i = 1} ^ { \infty } Z ^ { m } \rightarrow \mathcal { H }
\end{eqnarray}
which produces for every $\textbf{z}$ a classifier $\mathcal { Q } ( \mathbf { z } )$. The set $\mathcal{H}$ is called the hypothesis space.

Since we are going to discuss our classifier design method in the framework of kernel methods, which depends on a reproducing kernel Hilbert space associated with a Mercer \citep{cortes1995support}. Thus, the following definitions are needed.

\begin{definition}
A kernel function $K : X \times X \rightarrow \mathbb { R }$ is called \textbf{Mercer kernel} if it is continuous, symmetric and positive semidefinite, i.e., for any finite set of distinct points $\left\{ x _ { 1} ,\dots ,x _ { n } \right\} \subset X$, the matrix $\left( K \left( x _ { i } ,x _ { j } \right) \right) _ { i ,j = 1} ^ { n }$ is positive semidefinite. 
\end{definition}

\begin{definition}
The Hilbert space $\mathcal{H}_{ K }$ is called \textbf{Reproducing Kernel Hilbert Space (RKHS)} associated with the kernel $K$ if it is the linear span of the set of functions
$\left\{ K _ { x } : = K ( x ,\cdot ) : x \in X \right\}$ with the inner product $\langle \cdot ,\cdot \rangle_{\mathcal{H}_{ K }} = \langle \cdot ,\cdot \rangle _{K}$ satisfying

\begin{eqnarray}\notag
\| \sum _ { i = 1} ^ { m } c _ { i } K _ { x _ { i } } \| _ { K } ^ { 2} = \left\langle \sum _ { i = 1} ^ { m } c _ { i } K _ { x _ { i } } ,\sum _ { j = 1} ^ { m } c _ { j } K _ { x _ { j } } \right\rangle _ { K } = \sum _ { i ,j = 1} ^ { m } c _ { i } K \left( x _ { i } ,x _ { j } \right) c _ { j }
\end{eqnarray}
The reproducing property is given by
\begin{eqnarray}\label{1}
< K _ { x } ,g > _ { K } = g ( x ) ,\quad \forall x \in X ,g \in \mathcal { H } _ { K }
\end{eqnarray}
\end{definition}
Denote $C(X)$ as the space of continuous functions on $X$ with the norm$\| \cdot \| _ { \infty }$, then Eq. (\ref{1}) leads to 
$\| g \| _ { \infty } \leq \kappa \| g \| _ { K } , \forall g \in \mathcal { H } _ { K }$, where $\kappa = \operatorname{sup} _ { x \in X } \sqrt { K ( x ,x ) }$. This means $\mathcal { H } _ { K }$ can be embedded into $C(X)$.

\begin{definition} 
	The sign function is defined as: 
	\begin{eqnarray}\notag
\operatorname { sgn }  ( f ) ( x ) = \left\{ \begin{array} { l l } { 1,} & { \text{ if } f ( x ) \geq 0} \\ { - 1,} & { \text{ if } f ( x ) < 0} \end{array} \right.
	\end{eqnarray}
\end{definition}

\begin{definition} \citep{wu2006analysis}\label{definition4}
The \textbf{misclassification error} for a classifier $\operatorname { sgn } (f)$ induced from $f : X \rightarrow \{ 1,- 1\}$ is defined to be the probability of the event $f ( x ) \neq y$:

\begin{eqnarray}\notag
\mathcal { R } (\operatorname { sgn } ( f )) = \operatorname{Pr} \{ \operatorname { sgn } (f) ( \mathcal { X } ) \neq \mathcal { Y } \} = \int _ { X } \Pr ( \mathcal { Y } \neq \operatorname { sgn } (f) ( x ) | x ) d \rho _ { X } ( x )
\end{eqnarray}
Here $\rho _ { X }$ is the marginal distribution of $\rho$ on $X$ and $\rho ( \cdot | x ) = P ( \cdot | x )$ is the conditional probability measure given $\mathcal { X } = x$.
\end{definition}

\begin{definition}\citep{wu2006analysis}
A classification algorithm $\mathcal { Q }$ is said to be \textbf{Bayes-risk consistent}(with $\rho$) if $\mathcal { R } ( \mathcal { Q } ( \mathbf { z } ) )$ converges to $\mathcal { R } \left( f _ { c } \right)$ in probability, i.e., for every $\epsilon > 0$, 
\begin{eqnarray}\notag
\lim _ { m \rightarrow \infty } \operatorname{Pr} \left\{ \mathbf { z } \in Z ^ { m } : \mathcal { R } ( \mathcal { Q } ( \mathbf { z } ) ) - \mathcal { R } \left( f _ { c } \right) > \epsilon \right\} = 0.
\end{eqnarray}
\end{definition}
where 
\begin{eqnarray}\notag
f_{c}(x)=
\left\{
\begin{array}{cl}
+1 & \text{if~}  \Pr \left(\mathcal {Y}=1~|~\mathcal { X }=x\right) \geq \Pr(\mathcal {Y}=-1~|~\mathcal { X }=x) \\
-1 & \text{otherwise}
\end{array}
\right.
\end{eqnarray}
is called the Bayes rule.

\begin{definition}\label{pi}
The projection operator $\pi$ is defined on the space of measurable functions $f : X \rightarrow \mathbb { R }$ as 
\begin{eqnarray}\notag
\pi ( f ) ( x ) = \left\{ \begin{array} { l l } { 1 , } & { \text { if } f ( x ) \geq 1 } \\ { - 1 , } & { \text { if } f ( x ) \leq - 1 } \\ { f ( x ) , } & { \text { if } - 1 < f ( x ) < 1 } \end{array} \right.
\end{eqnarray}
\end{definition}	

\subsection{Standard SVMs}
SVM is a kind of kernel-based black-box modeling method, the main idea of which is to construct a hyperplane in an imaginary high-dimensional feature space that could separate two different
classes (labeled by the output $y=+1$ or $y=-1$) as far as possible \citep{cortes1995support}.
Generally, the $\ell$-norm soft margin SVM depending on a reproducing kernel Hilbert space associated with a Mercer kernel can be defined as follows
\begin{eqnarray}\label{original SVM}
&\displaystyle\ & \min_{f \in  { \mathcal { H } }_{K},\xi_{i}\in \mathbb{R}}\ \ \frac { 1} { 2} \| f  \| _ { K } ^ { 2}+\frac { C } { m } \sum _ { i = 1} ^ { m } \xi _ { i }^{\ell}\notag\\
&\text{s.t.} \ \ &y_if(\textbf{x}_i)\geq
1- \xi _ { i } ,\xi _ { i } \geq 0,\text{ for } i = 1,\dots ,m
\end{eqnarray}
Here, $C > 0$ is a trade-off parameter which may depend on $m$ and $\ell$ is a positive integer. 

Some theoretical aspects of SVM models have comprehensively analyzed in \citep{steinwart2002support},\citep{zhang2004statistical},\citep{wu2006analysis}, and \citep{chen2004support}.

Chen et al. \citep{chen2004support} and Wu et al. \citep{wu2006analysis} have verified the Bayes-risk consistency property of the SVM classifiers (denoted as $f _ { \mathbf { s } } $ below).  

\begin{theorem}	\citep{chen2004support},\citep{wu2006analysis}\label{theorem7}
If $K$ is the Gaussian kernel,
$K ( x , y ) = \exp \left\{ - \frac { | x - y | ^ { 2 } } { \sigma ^ { 2 } } \right\}$,
and some other conditions are satisfied, then with probability at least $1-\delta$ there holds
\begin{eqnarray}\notag
\mathcal { R } \left( \operatorname{sgn} \left( f _ { \mathbf { s } } \right) \right) - \mathcal { R } \left( f _ { c } \right) \leq \mathcal { O }\left( ( \log m ) ^ { - s } \right).
\end{eqnarray}
where $f_c$ is the best classifier, Bayes rule, and $\mathcal{R}(\cdot)$ is the misclassification error.
\end{theorem}
The misclassification error can be reformulated as follows
\begin{eqnarray}\label{3}
\mathcal { R }(f)=\mathbb{ E }_{\mathcal { X }} \left\{ [ 1- \Pr \left(\mathcal { Y }=1~|~\mathcal { X }\right) ] \mathbf { 1} ( f ( \mathcal { X } ) = 1) +  \Pr \left(\mathcal { Y }=1~|~\mathcal { X }\right) \mathbf { 1} ( f ( \mathcal { X } ) = - 1) \right\}.
\end{eqnarray}
Here $f$ is a classifier,$f : X \rightarrow \{ 1,- 1\}$ and $\mathbf {1} \left(\cdot\right)$ is the indicator function: it assumes the value 1 if its argument is true, and 0 if otherwise.

Theorem \ref{theorem7} demonstrates that the SVM classifier is asymptotically equivalent to the best classifier, Bayes rule. Because $f_c$ minimizes the misclassification error above, thus, it also indicates that the SVM classifier is essentially to minimize the target $\mathcal{ R }(f)$ given $f$ belongs to RKHS.

To distinguish the model (\ref{original SVM}) from other varieties of SVMs, without loss of generality, we will describe the model (\ref{original SVM}) with $\ell=1$ as the \textbf{original SVM}.

\subsection{ Extensions of SVMs}

As demonstrated above that the original SVM can be interpreted as minimizing the misclassification error (or maximizing accuracy), when data size grows to indefinitely. Meanwhile, there are many varieties of SVMs built for incorporating prior knowledge to satisfy different purposes. We will illustrate several examples which  serve as an excellent  source of inspiration for building our learning target and learning scheme to achieve it in this work.

Lin et al. \citep{lin2002support} considered the non-standard situation by taking unequal misclassification costs (without consideration of sampling bias here) into account. Then they built a classification model (\ref{svm2}) by assuming the costs for false positive and false negative are $c ^ { + }$ and $c ^ { - }$, respectively. 

\begin{eqnarray}\label{svm2}
&\displaystyle\ & \min_{f \in { \mathcal { H } }_{K},\xi_{i}\in \mathbb{R}}\ \ \frac { 1} { 2C} \| f \| _ { K } ^ { 2}+\frac { 1} { m } \sum _ { i = 1} ^ { m }L(y_{i}) \xi _ { i }\notag\\
&\text{s.t.} \ \ &y_if({x}_i)\geq
1- \xi _ { i } ,\xi _ { i } \geq 0,\text{ for } i = 1,\dots ,m
\end{eqnarray}
Here, $L(\cdot)$ is a two-valued function where $L(-1)=c^+$ and $L(1)=c^-$, respectively.  

Actually, they illustrated that the solution of model (\ref{svm2}) approaches a classifier which minimizes the expected cost below
\begin{eqnarray}\label{4}
\mathbb{E}_{\mathcal { X }} \left\{ c ^ { + } [ 1- \Pr \left(\mathcal { Y }=1~|~\mathcal { X }\right) ] \mathbf { 1} ( f ( \mathcal { X } ) = 1) + c ^ { - } \Pr \left(\mathcal { Y }=1~|~\mathcal { X }\right) \mathbf { 1} ( f ( \mathcal { X } ) = - 1) \right\}
\end{eqnarray}

Notice that Eq. ({\ref{4}}) indicates the implied learning target of model (\ref{svm2}). The expected cost Eq. (\ref{4}) and the misclassification error Eq. (\ref{3}) reflect the essential difference between the model (\ref{svm2}) and the original SVM. Moreover, Eq. ({\ref{4}}) can be understood easier for humans compared to the optimization problem (\ref{svm2}).

Wu et al. \citep{wu2004incorporating} proposed another generalization of SVMs (\ref{svm3}) that permits the incorporation of prior knowledge. Concretely, they assigned each training sample $\left( x _ { i } ,y _ { i } \right) $ to a confidence value $v _ { i } \in (0,1]$ which indicates the confidence level of $y_i$'s labeling. Then, they built a classification model(omitting the margin normalization term here) to include this knowledge properly as below
\begin{eqnarray}\label{svm3}
&\displaystyle\ & \min_{f \in { \mathcal { H } }_{K},\xi_{i}\in \mathbb{R}}\ \ \frac { 1} { 2C} \| f \| _ { K } ^ { 2}+\frac { 1 } { m } \sum _ { i = 1} ^ { m }h(v _ { i }) \xi _ { i }\notag\\
&\text{s.t.} \ \ &y_if({x}_i)\geq
1- \xi _ { i } ,\xi _ { i } \geq 0,\text{ for } i = 1,\dots ,m
\end{eqnarray}
where $h$ is a monotonically increasing fucntion.
As illustrated in \citep{wu2004incorporating}, Model (\ref{svm3}) is built for finding a decision plane to which samples with high confidence label contribute more. Note that the models (\ref{svm3}) and (\ref{svm2}) have a similar form with each other. Thus we believe there is a measure similar to the expected cost above to depict the essential learning goal of the model (\ref{svm3}).
 
Mangasarian et al. \citep{mangasarian2008nonlinear} also presented a classification model with knowledge incorporated where prior information with the form
\begin{equation}\label{pr3}
g ( {x} ) \leq 0\Rightarrow f(x) \geq \alpha,~~\forall~
{x}\in \Gamma
\end{equation}
is converted into a linear constraint
\begin{equation}\label{prior}
f(x) - \alpha + vg \left({x} \right) + \zeta _ { i } \geq 0, ~~\forall~
{x}\in \Gamma.
\end{equation}	
The derived classification model \citep{mangasarian2008nonlinear} can be rewritten as follows 
\begin{eqnarray}\label{generalized svm4}
\begin{array}{cl}
\min_{f\in{\mathcal{H}}_{ K }} & \frac{1}{2C} \|f\|^{2}_{K}+\frac{1}{m}\left(\sum _ { i = 1} ^ { m }\xi_{i}+\sum _ { j = 1} ^ { p }\zeta_{j}\right) \\
\text{s.t.} & y_{j}f(x^{j})\geq 1- vg\left({x} ^ { j } \right)-\zeta_{j}, ~j=1,2,...,p\\
&y_{i}f(x_{i})\geq 1-\xi_{i},i=1,2,...,m\\
& \xi_{i}\geq 0, ~\zeta_{j}\geq 0, \quad i=1,2,...,m
\end{array} 
\end{eqnarray}
if we assume $f\in\mathcal{H}_{ K }$, $\alpha=1$ and the function $g(x)$ is convex. Notice that ${x} ^ { 1} ,{x} ^ { 2} ,\dots ,{x} ^ { p }$ is the discretized sample points in the region $g(x)\leq0$ and $v\geq0$.

Observe that the model (\ref{generalized svm4}) is quite different from the models (\ref{svm2}) and (\ref{svm3}), which dues to the different types of prior knowledge incorporated. Actually, the prior knowledge Eq. (\ref{pr3}) implies the geometric information of the decision plane, which aims at improving the performance of the model (\ref{generalized svm4}) \citep{mangasarian2008nonlinear}. However, we can not directly find the essential learning target of it.

Motivated from the above analysis, we devote to design a classifier whose performance can be theoretically guaranteed in this paper. Generally, we hope to design a knowledge incorporation learning scheme such that the derived classifier can make the safe prediction through combining some human experience properly. Therefore, we first choose to integrate prior knowledge into our learning target(the utility function, which we think can be understood easier for users) suitably, then derive a classifier to approach the optimizer of this learning target by learning from training data.




 
\section{Classifier Design Based on Utility Function}

In this part, we will demonstrate the framework detailedly about how to design the knowledge-based SVM classifier to optimize the utility function where some prior knowledge is incorporated.

Firstly, we will illustrate the type of prior knowledge used in this paper. Secondly, we try to quantify the goal to build our classifier, the utility function. Finally, we prove our main result that the proposed knowledge-based SVM classifier maximizes the utility function, by verifying the (asymptotic) equivalence between the proposed knowledge-based SVM classifier and a generalization of the Bayes rule, which maximizes the utility function.


\subsection{Prior Knowledge}

In machine learning fields, prior knowledge can refer to any known information about or related to the concerning objects, such as data, knowledge, specifications, etc, \citep{qu2011generalized}. Specifically, there are two general types pf prior knowledge in classification problems: class-invariance and knowledge on the data \citep{lauer2008incorporating}. 

In this work, attention is mainly focused on the second type of prior information, in particular, we choose to integrate the prior knowledge on \textbf{positive class} samples into black-box models as follows
\begin{equation}\label{prior1}
g^+({x})\leq 0 \Longrightarrow  \text{The cost to misclassify sample }x ~\text{is} ~\hat{c},~~\forall~ {x}\in \Gamma^+
\end{equation}
where $\Gamma^+ \subseteq
\mathbb{R}^n$, $g^+:\Gamma^+ \to \mathbb{R}$, and $\hat{c}>1$. This kind of prior knowledge Eq. (\ref{prior1}) indicates that misclassifying positive class samples in the region of $\mathcal{A}_{+}=\left\{x|g^+({x})\leq{0}\right\}$ are more severe than the samples in other regions. Here, $\hat{c}$ can take different values according to different learning targets in applications.

Considering prior knowledge in the form of logical implication in Eq. (\ref{prior1}) is of practical significance and which we think can obtain easily in real applications. For example, missing the disease diagnosis of a patient whose physical examination items are too far below or above the safe ranges can lead to serious consequences. Therefore, intuitively, it is reasonable to impose much more punishments on the misclassified samples lie in the unsafe region, i.e., $\mathcal{A}_{+}$.  Here, the presented prior knowledge Eq. (\ref{prior1}) can be understood as a piece of domain experts experience. 

Without loss of generality, we will only consider to integrate Eq. (\ref{prior1}) (called as \textbf{positive class prior knowledge} for short) in this paper. There may be other kinds of prior knowledge, such as, the negative class prior knowledge, which can be incorporated into our learning framework by following standard techniques below.

\subsection{The Utility Function Based on Prior Knowledge}

In this part, we will formally present the definition of the utility function which quantifies a new learning target in this paper.

As described above, we attempt to build a classifier which minimizes the risks estimated by domain experts in applications. In other words, our learning target should include prior knowledge provided by Eq. (\ref{prior1}) suitably, so that the domain experts can interact with the learning process and achieve their goals.

To improve the generality of our learning target, and motivated by the work of Lin et al. \citep{lin2002support} where costs for misclassifying positive and negative classes($c^-$ and $c^+$) are unequal. We define the utility function as the expectation of negative ensemble risks, which includes the risks of misclassifying positive class samples, misclassifying negative class samples, and misclassifying positive class samples in the region of $\mathcal{A}_{+}$.


\begin{definition}
The \textbf{utility function}, which depends on the classifier function $\tilde{f}$, is defined as the expectation of negative ensemble risks 
\begin{eqnarray}\label{utility}
\mathcal { U} ( \tilde{f} )&=&
-\mathbb{E}_{\mathcal { X }}\Bigg\{\Pr \left(\mathcal {Y}=-1~|~\mathcal { X }\right)\mathbf { 1}(\tilde{f}(\mathcal { X })=1)+\\ &&\left[
\frac{c^{-}\hat{c}}{c^{+}} \Pr \left(\mathcal {Y}=1~|~\mathcal { X }\right) \mathbf { 1}(\mathcal { X }\in\mathcal{A}_{+})
+\frac{c^{-}}{c^{+}} \Pr \left(\mathcal {Y}=1~|~\mathcal { X }\right) \mathbf { 1}(\mathcal { X }\in\bar{\mathcal{A}}_{+})
\right] \mathbf { 1}(\tilde{f}(\mathcal { X })=-1)\notag
\Bigg\},
\end{eqnarray}
where $\bar{\mathcal{A}}_{+}$ is complementary set of $\mathcal{A}_{+}$, that is, $\bar{\mathcal{A}}_{+}=\left\{x|g^+({x})> {0}\right\}$.
\end{definition}
Notice that there are three parameters $c^{+}>0, c^{-}>0$ and $\hat{c}>0$ in the formula of the utility function, which represent the costs for false positive, false negative and for false negative in the region $\mathcal{A}_{+}$.
 Here, the assumption $\hat{c}>1$ can be understood as the designed classifier is preferable to classify the sample points lie in the region $\mathcal{A}_{+}$ as positive samples so that the risk of missing diagnosis is reduced. For example, in the process of diagnosing diseases, doctors are inclined to judge the patients whose physical examination items are abnormal as getting the disease.
The utility function reduces to the expected cost Eq. (\ref{4}) proposed by Lin et al. \citep{lin2002support} in the situation of $\hat{c}=1$.

The presented utility function can be seen as an extension of the traditional performance measure of the classifier, such as the misclassification error \citep{wu2006analysis}, \citep{chen2004support} and expected cost \citep{lin2002support}. Also, it is worth to notice that the created utility function has been embedded into prior knowledge both from label and feature aspects. Meanwhile, the proposed utility function allows users to tune the parameters of risk cost to achieve their learning targets. 

 If the probability function $\Pr \left(Y=1~|~\mathcal { X }=x\right)$ is available, we are able to identify the \textbf{optimal classifier} that maximizes the utility function Eq. (\ref{utility}) as follow
 \begin{eqnarray}\label{modified bayes}
 f_{q}(x)=
 \left\{
 \begin{array}{cl}
 +1 & \text{if~} \frac{c^{-}\hat{c}}{c^{+}} \Pr \left(\mathcal {Y}=1~|~\mathcal { X }=x\right) \mathbf {1}(x\in\mathcal{A}_{+})
 +\frac{c^{-}}{c^{+}} \Pr \left(\mathcal {Y}=1~|~\mathcal{X}=x\right) \mathbf { 1}(x\in\bar{\mathcal{A}}_{+}) \\
 & ~~\geq \Pr(\mathcal {Y}=-1~|~\mathcal { X }=x) \\
 -1 & \text{otherwise}
 \end{array}
 \right.
 \end{eqnarray}

However, in general the probability function $\Pr \left(Y=1~|~\mathcal { X }=x\right)$ is unknown. Thus, in the next section, we are going to build a classification learning framework to approximate the optimal classifier Eq. (\ref{modified bayes}) by learning from the training set.
 

\subsection{Optimal Utility Function based Classifier Design}

In this part, we try to build a classification model which yields a classifier to approach the optimal classifier Eq. (\ref{modified bayes}), which is motivated by the Bayes-risk consistency property of original SVMs in RKHS \citep{aronszajn1950theory}, \citep{steinwart2002support}, \citep{zhang2004statistical}, \citep{chen2004support} and \citep{wu2006analysis}.

Thus, we firstly attempt to construct a loss function such that the corresponding generalization error can be minimized on the optimal classifier $f_q$ Eq. (\ref{modified bayes}). Then, we build a classification model to approximate the generalization error by learning from the training data.



The following theorem relates the generalization error to our utility function Eq. (\ref{utility}) through the optimal classifier $f_q$ Eq. (\ref{modified bayes}).

\begin{theorem}{\label{theorem 1}}
The minimizer of $\mathbb{E}[V(\mathcal {Y},f(\mathcal {X}))]=\int _ { Z } V ( y ,f ( x ) ) d \rho ( x ,y )$ is $f_{q}$, where the \textbf{loss function} is defined as
\begin{eqnarray}\label{piecewise loss function}
V(y,f(x))=\left\{
\begin{matrix}
\frac{c^{-}\hat{c}}{c^{+}}(1-yf(x))_{+}, & \text{if~} x\in \mathcal{A_{+}} \text{~and~} y=1 \\
\frac{c^{-}}{c^{+}}(1-yf(x))_{+}, & \text{otherwise}
\end{matrix}
\right.
\end{eqnarray}
and $(1-yf(x))_{+}=\max(0,1-yf(x))$.
\end{theorem}

\begin{proof}
Notice 
\begin{eqnarray}\notag
\mathbb{E}\left\{
V(\mathcal {Y},f(\mathcal {X}))
\right\}=\mathbb{E}_{\mathcal{X}}\left\{\mathbb{E}_{\mathcal{Y}}\left\{
V(\mathcal {Y},f(\mathcal {X}))~|~\mathcal {X}=x
\right\}\right\}.
\end{eqnarray}
We can minimize $\mathbb{E}\left\{
V(\mathcal {Y},f(\mathcal {X}))
\right\}$ by minimizing $\mathbb{E}_{\mathcal{Y}}\left\{
V(\mathcal {Y},f(\mathcal {X}))~|~\mathcal {X}=x
\right\}$, for every fixed $x$.
Because 
\begin{eqnarray}
&&\mathbb{E}_{\mathcal{Y}}\left\{
V(\mathcal {Y},f(\mathcal {X}))~|~\mathcal {X}=x
\right\} \notag \\
&&=
\left\{
\begin{matrix}
\frac{c^{-}\hat{c}}{c^{+}}\Pr(\mathcal {Y}=1~|~\mathcal {X}=x)(1-f(x))_{+}+\Pr(\mathcal {Y}=-1~|~X=x)(1+f(x))_{+}, & \text{if~} x\in\mathcal{A}_{+}\\
\frac{c^{-}}{c^{+}}\Pr(\mathcal {Y}=1~|~X=x)(1-f(x))_{+}+\Pr(\mathcal {Y}=-1~|~\mathcal {X}=x)(1+f(x))_{+}, & \text{otherwise}
\end{matrix}
\right. \notag
\end{eqnarray}
Then, for any fixed $x$, we can search for $w=f(x)$ that minimizes $G(w)=\mathbb{E}_{\mathcal{Y}}\left\{
V(\mathcal {Y},w)\right\} $.

Notice that the minimizer of $G(w)$ must be in $[-1,1]$. For any $w$ outside $[-1,1]$, let $w^\prime=\operatorname{sgn}(w)$, then $w^\prime$ is in $[-1,1]$ and it is easy to check $G(w^\prime)<G(w)$. Thus, we can restrict $w\in[-1,1]$, there holds:

if $x\in\mathcal{A}_{+}$, $G(w)$ is minimized at $w=1$, when 
$\frac{c^{-}\hat{c}}{c^{+}}\Pr(\mathcal {Y}=1~|~\mathcal {X}=x) \geq \Pr(\mathcal {Y}=-1~|~\mathcal {X}=x)$; And at $w=-1$, when $\frac{c^{-}\hat{c}}{c^{+}}\Pr(\mathcal {Y}=1~|~\mathcal {X}=x) < \Pr(\mathcal {Y}=-1~|~\mathcal {X}=x)$.

Otherwise, $G(w)$ is minimized at $w=1$, when 
$\frac{c^{-}}{c^{+}}\Pr(\mathcal {Y}=1~|~\mathcal {X}=x) \geq \Pr(\mathcal {Y}=-1~|~\mathcal {X}=x)$; And at $w=-1$, when $\frac{c^{-}}{c^{+}}\Pr(\mathcal {Y}=1~|~\mathcal {X}=x) < \Pr(\mathcal {Y}=-1~|~\mathcal {X}=x)$. 

Combining above two results yields that the minimizer of $\mathbb{E}[V(\mathcal {Y},f(\mathcal {X}))]$ is $f_q$.
Thus the theorem is proved.
\end{proof} 

\begin{definition}
	We denote the \textbf{empirical error} related to loss function $V(y,f(x))$ as:
	$\mathcal { E }_ { \mathbf { z } } ( f ) = \frac { 1} { m } \sum _ { i = 1} ^ { m } V \left( y _ { i } ,f \left( x _ { i } \right) \right)$.
\end{definition}	

\begin{definition}
	We denote the \textbf{generalization error} related to loss function $V(y,f(x))$ as:
	$\mathcal { E } ( f )=\mathbb{E}[V(\mathcal {Y},f(\mathcal {X}))]$.
\end{definition}

Theorem \ref{theorem 1} indicates that if we can design a classifier to minimize the empirical error, then the classifier also minimizes the generalization error by law of large numbers.

In order to minimize the empirical error and borrow the idea of maximum-margin from original SVMs, we produce the classification model as below
\begin{eqnarray}\label{modify svm1}
\min_{f\in{\mathcal{H}}_{K}} \quad 
\frac{1}{2C} \|f\|^{2}_{K}
+\frac{1}{m}\left(\sum_{x_{i}\in\mathcal{A}_{+},~y_{i=1}} \frac{c^{-}\hat{c}}{c^{+}}(1-y_{i}f(x_{i}))_+
+\sum_{\text{otherwise}} \frac{c^{-}}{c^{+}}(1-y_{i}f(x_{i}))_+\right)	
\end{eqnarray}	

Minimizing model (\ref{modify svm1}) can be rewritten as a constrained optimization problem with a differentiable objective function in the following way. 
	
For each $i \in \left\{1,\dots ,m\right\}$, we introduce a variable $\xi_{ i }=\frac{c^{-}\hat{c}}{c^{+}}\max\{0,1-y_{i}f(x_{i})\}$ if $x_{i}$ belongs to $\mathcal{A}_{+}$ and its label is $+1$, otherwise the other type of variable is introduced as $\zeta_{i}=\frac{c^{-}}{c^{+}}\max\{0,1-y_{i}f(x_{i})\}$.

Thus we can rewrite the optimization problem as follows
\begin{eqnarray}\label{modified svm2}
\begin{array}{cl}
\min_{f\in{\mathcal{H}}_{ K }} & \frac{1}{2C} \|f\|^{2}_{K}+\frac{1}{m}\left(\sum_{x_{i}\in\mathcal{A}_{+},~y_{=1}}\xi_{i}+\sum_{\text{otherwise}}\zeta_{i}\right) \\
\text{s.t.} & y_{i}f(x_{i})\geq 1-\frac{c^{+}}{c^{-}\hat{c}}\xi_{i},\quad x_{i}\in\mathcal{A}_{+} \text{ and } y_{i}=1\\
& y_{j}f(x_{j})\geq 1-\frac{c^{+}}{c^{-}}\zeta_{j}, ~\quad \text{otherwise}\\
& \xi_{i}\geq 0, ~\zeta_{j}\geq 0, ~\quad \qquad i=1,2,...,m_1;j=1,2,...,m_2;\\&
m_1+m_2=m.
\end{array} 
\end{eqnarray}

Intuitively, the produced model (\ref{modified svm2}) implies that misclassifying the positive class sample lie in the region $\mathcal{A}_{+}$ will suffer much more losses, which is coincident with the meaning of the utility function Eq. (\ref{utility}).

In the next section, we will strictly verify that the proposed knowledge-based classifier converges to the optimal classification rule Eq. (\ref{modified bayes}) with respect to the utility function Eq. (\ref{utility}), by using techniques from \citep{zhang2004statistical}, \citep{wu2006analysis}, and \citep{fan2017learning}. For simplicity, we will denote the presented knowledge incorporated classification model (\ref{modified svm2}) as the \textbf{knowledge-based SVM}.

\subsection{Error Analysis of Knowledge-based SVM}

The framework to prove the convergence property of the proposed \textbf{knowledge-based SVM} classifier is organized as follows:
Firstly, the bridge between $\mathcal { U} (\operatorname { sgn }(f))$ and $\mathcal { E } ( f )=\mathbb{E}[V(\mathcal {Y},f(\mathcal {X}))]$ is established;
Secondly, the upper bound of $\mathcal { U} \left( \operatorname { sgn} (f _ { \mathbf { z } }) \right) - \mathcal { U } \left( f _ { q } \right)$ is given, which can be decomposed into two terms corresponding to sample error and regularization error in \citep{wu2006analysis}; Finally, we will estimate these two terms respectively.

The bridge between $\mathcal { U} (\operatorname { sgn }(f))$ and $\mathcal { E } ( f )=\mathbb{E}[V(\mathcal {Y},f(\mathcal {X}))]$ is illustrated by the following theorem.
\begin{theorem}\label{theorem2}
For any function $f$ mapping from $X$ to $\mathbb{R}$ the relationship between utility function and generalization error can be illustrated as follows: 
\begin{eqnarray}\notag
\mathcal { U } \left( f _ { q } \right)-\mathcal { U} \left( \operatorname { sgn } (f) \right)\leq \mathcal { E} \left( f \right) - \mathcal { E } \left( f _ { q } \right).
\end{eqnarray}
\end{theorem}

\begin{proof}
First denote
	\begin{equation*}
		\phi(\mathcal{X},\mathcal{ A }_{+},\bar{\mathcal{ A }}_{+})=
		\frac{c^{-}\hat{c}}{c^{+}} \Pr \left(\mathcal{Y}=1~|~\mathcal{X}\right) \mathbf { 1}(\mathcal{X}\in\mathcal{A}_{+})
		+\frac{c^{-}}{c^{+}} \Pr \left(\mathcal{Y}=1~|~\mathcal{X}\right) \mathbf { 1}(\mathcal{X}\in\bar{\mathcal{A}}_{+})
	\end{equation*}
	\begin{equation*}
	\mathbf{I}(\operatorname { sgn } (f)(\mathcal{X}),f_{q}(\mathcal{X}))= \mathbf{1}(\operatorname { sgn } (f)(\mathcal{X})=-1,f_{q}(\mathcal{X})=1)-\mathbf { 1}(\operatorname { sgn } (f)(\mathcal{X})=1,f_{q}(\mathcal{X})=-1).
	\end{equation*}
According to the definition of the utility function in Eq. (\ref{utility}), there holds
\begin{eqnarray}\notag
&&\mathcal { U } \left( f _ { q } \right)-\mathcal { U} \left( \operatorname { sgn } (f) \right)\\
&=&
\mathbb{E}_{\mathcal{X}}\notag
\left\{\Big[
\phi(\mathcal{X},\mathcal{ A }_{+},\bar{\mathcal{ A }}_{+})
-\Pr \left(\mathcal{Y}=-1~|~\mathcal{X}\right)
\Big]
\mathbf{1}(\operatorname { sgn } (f)(\mathcal{X})=-1,f_{q}(\mathcal{X})=1)
\right\}\\
&&+\mathbb{E}_{\mathcal{X}}\notag
\left\{ \Big[
\Pr \left(\mathcal{Y}=-1~|~\mathcal{X}\right)
-\phi(\mathcal{X},\mathcal{ A }_{+},\bar{\mathcal{ A }}_{+})
\Big]
\mathbf { 1}(\operatorname { sgn } (f)(\mathcal{X})=1,f_{q}(\mathcal{X})=-1)
\right\}\\
&&= \left\{\notag
\begin{matrix}\Big[\frac{c^{-}\hat{c}}{c^{+}} \Pr \left(\mathcal{Y}=1~|~\mathcal{X}\right) -\Pr \left(\mathcal{Y}=-1~|~\mathcal{X}\right)\Big]\mathbf{I}(\operatorname { sgn } (f)(\mathcal{X}),f_{q}(\mathcal{X}))

& \text{if~} \mathcal {X}\in \mathcal{A}_{+}\\
\Big[\frac{c^{-}}{c^{+}} \Pr \left(\mathcal{Y}=1~|~\mathcal{X}\right) -\Pr \left(\mathcal{Y}=-1~|~\mathcal{X}\right)\Big]\mathbf{I}(\operatorname { sgn } (f)(\mathcal{X}),f_{q}(\mathcal{X}))
& \text{if~} \mathcal {X}\in \bar{\mathcal{A}}_{+}
\end{matrix}
\right.
\end{eqnarray}

Notice that there are only two possible values of term  $\mathbf{I}(\operatorname { sgn } (f)(\mathcal{X}),f_{q}(\mathcal{X}))$,
specifically, if it equals to $1$, which means the argument $\operatorname { sgn }(f)(\mathcal{X})=-1,f_{q}(\mathcal{X})=1$ is true, there holds
\begin{eqnarray}\label{equation 28}
\mathcal { U } \left( f _ { q } \right)-\mathcal { U} \left( \operatorname { sgn } (f) \right)=
\left\{
\begin{matrix}
\mathbb{E}_{\mathcal{X}}\left\{\frac{c^{-}\hat{c}}{c^{+}} \Pr \left(\mathcal{Y}=1~|~\mathcal{X}\right) -\Pr \left(\mathcal{Y}=-1~|~\mathcal{X}\right)
\right\}
& \text{if~} \mathcal{X}\in \mathcal{A}_{+}\\
\mathbb{E}_{\mathcal{X}}\left\{\frac{c^{-}}{c^{+}} \Pr \left(\mathcal{Y}=1~|~\mathcal{X}\right) -\Pr \left(\mathcal{Y}=-1~|~\mathcal{X}\right)
\right\}
& \text{if~} \mathcal{X}\in \bar{\mathcal{A}}_{+}
\end{matrix}
\right.
\end{eqnarray}
Because 
\begin{eqnarray}\notag
\mathbb{E}\left\{
V(\mathcal {Y},f(\mathcal {X}))
\right\}=\mathbb{E}_{\mathcal{X}}\left\{\mathbb{E}_\mathcal{Y}\left\{
V(\mathcal {Y},f(\mathcal {X}))~|~\mathcal {X}
\right\}\right\}.
\end{eqnarray}
Through simple computations there holds
\begin{eqnarray}\label{equation 30}
&&\mathbb{E}_\mathcal{Y}\left\{
V(\mathcal {Y},\pi(f)(\mathcal {X}))~|\mathcal {X}
\right\} - \text{E}_\mathcal{Y}\left\{
V(\mathcal {Y},f_q(\mathcal {X}))~|\mathcal {X}
\right\}\\
&&= \left\{\notag
 \begin{matrix}
 \left\{\frac{c^{-}\hat{c}}{c^{+}} \Pr \left(\mathcal{Y}=1~|~\mathcal{X}\right) -\Pr \left(\mathcal{Y}=-1~|~\mathcal{X}\right)
 \right\}(f_q(\mathcal{ X })-\pi(f)(\mathcal{ X }))
 & \text{if~} \mathcal {X}\in \mathcal{A}_{+}\\
 \left\{\frac{c^{-}}{c^{+}} \Pr \left(\mathcal{Y}=1~|~\mathcal{X}\right) -\Pr \left(\mathcal{Y}=-1~|~\mathcal{X}\right)
 \right\}(f_q(\mathcal{ X })-\pi(f)(\mathcal{ X }))
 & \text{if~} \mathcal {X}\in \bar{\mathcal{A}}_{+}
 \end{matrix}
 \right.
\end{eqnarray}
where $\pi$ is the projection operator as defined in Definition \ref{pi}.

Consider that $\operatorname{sgn}(f)(\mathcal{X})=-1,f_{q}(\mathcal{X})=1$, then comparing Eq. (\ref{equation 28}) to Eq. (\ref{equation 30}), there holds the inequality:
$\mathcal { U } \left( f _ { q } \right)-\mathcal { U} \left( \operatorname { sgn } (f) \right) \leq \mathcal { E} \left( \pi(f) \right) - \mathcal { E } \left( f _ { q } \right)$. Notice that $\mathcal { E} \left( \pi(f) \right)\leq\mathcal { E}(f)$, thus $\mathcal { U } \left( f _ { q } \right)-\mathcal { U} \left( \operatorname { sgn } (f) \right) \leq \mathcal { E} \left( f \right) - \mathcal { E } \left( f _ { q } \right)$. 

Similar to the discussions on the other case: $\operatorname{sgn}(f)(\mathcal{X})=1,f_{q}(\mathcal{X})=-1$.

Thus the conclusion holds.
\end{proof}	
	
\begin{definition}\label{definition12}
\begin{eqnarray}\notag
f _ { K ,C } = \operatorname{arg} \min _ { f \in \mathcal{ H } _ { K } } \left\{ \mathcal { E } ( f ) + \frac { 1} { 2C } \| f  \| _ { K } ^ { 2} \right\},
\end{eqnarray}
\end{definition}

\begin{definition}\label{definition14}
	\begin{eqnarray}\notag
	f _ { \mathbf { z } } = \operatorname{arg} \min _ { f \in { \mathcal{H} } _ { K } } \left\{ \mathcal { E } _ { \mathbf { z } } ( f ) + \frac { 1} { 2C } \| f  \| _ { K } ^ { 2} \right\},
\end{eqnarray}
\end{definition}	

The upper bound of $\mathcal { U} \left( \operatorname { sgn} (f _ { \mathbf { z } }) \right) - \mathcal { U } \left( f _ { q } \right)$ is decomposed into the sample error and regularization error below.
\begin{theorem}
For every $C>0$, there holds
\begin{eqnarray}\notag
\mathcal { U } \left( f _ { q } \right)-\mathcal { U} \left( \operatorname { sgn } (f) \right) \leq \mathcal { S } ( m ,C ) + \mathcal { D } ( C ),
\end{eqnarray}
where 

$\mathcal { S } ( m ,C ) : = \left\{ \mathcal { E } \left( f _ { \mathbf { z } } \right) - \mathcal { E } _ {\mathbf { z } } \left( f _ { \mathbf { z } } \right) \right\} + \left\{ \mathcal { E } _ { \mathbf { z } } \left( f _ { K ,C } \right) - \mathcal { E } \left( f _ { K ,C } \right) \right\}$.
and 

$\mathcal { D } ( C ) : = \operatorname{inf} _ { f \in { \mathcal { H } } _ { K } } \left\{ \mathcal { E } ( f ) - \mathcal { E } \left( f _ { q } \right) + \frac { 1} { 2C } \| f  \| _ { K } ^ { 2} \right\}$.
\end{theorem}

\begin{proof}
According to Theorem \ref{theorem2}, the equation $\mathcal { U} \left( \operatorname { sgn } (f_\mathbf{z}) \right) - \mathcal { U } \left( f _ { q } \right) \leq \mathcal { E} \left( f_\mathbf{z} \right) - \mathcal { E } \left( f _ { q } \right)$ holds.
So we can decompose the term $\mathcal { E } \left( f _ \mathbf{ z } \right) - \mathcal { E } \left( f _ { q } \right)$ by following the techniques in \citep{chen2004support} and \citep{wu2006analysis}, as below
\begin{align}
\mathcal { E } \left( f _ \mathbf{ z } \right) - \mathcal { E } \left( f _ { q } \right)=&
\left\{ \mathcal { E } \left( f _ { \mathbf { z } } \right) - \mathcal { E } _ { \mathbf { z } } \left( f _ \mathbf{ z } \right) \right\} + 
\Bigg\{\left( \mathcal { E } _ { \mathbf { z } } \left( f _ { \mathbf { z } } \right) + \frac { 1} { 2C } \| f _ { \mathbf { z } } \| _ { K } ^ { 2} \right)-\left( \mathcal { E } _ \mathbf{ z } \left( f _ { K ,C } \right) + \frac { 1} { 2C } \| f _ { K ,C } \| _ { K } ^ { 2} \right)\Bigg\} \notag\\
&+
\left\{ \mathcal { E } _ { \mathbf { z } } \left( f _ { K ,C } \right) - \mathcal { E } \left( f _ { K ,C } \right) \right\}+
\left\{ \mathcal { E } \left( f _ { K ,C } \right) - \mathcal { E } \left( f _ { q } \right) + \frac { 1} { 2C } \| f _ { K ,C } \| _ { K } ^ { 2} \right\} - \frac { 1} { 2C } \| f _ \mathbf{ z } \| _ { K } ^ { 2} \notag
\end{align}
By the definition of $f _ { \mathbf { z } }$ in Definition \ref{definition12}, the second term is $\leq 0$. By the definition of $f _ { K ,C }$ in Definition \ref{definition14}, the fourth term is $\mathcal { D } ( C )$. Thus the $\mathcal { E } \left( f _ {\mathbf { z } } \right) - \mathcal { E } \left( f _ { q } \right)$ can be bounded by $\mathcal { S } ( m ,C ) + \mathcal { D } ( C )$. Hence the desired result is proved.
\end{proof}

For simplicity, we can refer to the first term $\mathcal { S } ( m ,C )$ as sample error and the second term $\mathcal { D } ( C )$ as regularization error similar to the definitions in \citep{chen2004support} and \citep{wu2006analysis}.

Firstly, in order to give a upper bound of sample error, we need some technical lemmas below. 

\begin{definition}\citep{fan2017learning}
The function $F : \prod _ { k = 1} ^ { m } X _ { k } \rightarrow \mathbb { R }$  has bounded differences $\left\{ c _ { k } \right\} _ { k = 1} ^ { m }$, if for all $1\leq k \leq m$, 
\begin{eqnarray}
\left.\begin{array} { c } { \max } \\ { z _ { 1} ,\cdots ,z _ { k } ,z _ { k } ^ { \prime } \dots ,z _ { m } } \end{array} \right. &| F \left( z _ { 1} ,\cdots ,z _ { k - 1} ,z _ { k } ,z _ { k + 1} ,\cdots ,z _ { m } \right) - F \left( z _ { 1} ,\cdots ,z _ { k - 1} ,z _ { k } ^ { \prime } ,z _ { k + 1} ,\cdots ,z _ { m } \right) |\notag\\&\leq c _ { k }\notag
\end{eqnarray}
\end{definition}

The following probability inequality was motivated by \citep{mcdiarmid1989method} and will be used to estimate our sample error.
\begin{lemma}
(\textbf{McDiarmid’s inequality}) Suppose $F : \prod _ { k = 1} ^ { m } X _ { k } \rightarrow \mathbb { R } $ has bounded differences $\left\{ c _ { k } \right\} _ { k = 1} ^ { m }$, then for all $\epsilon > 0$, there holds
\begin{eqnarray}\notag
\operatorname{Pr} \{ F ( \mathbf { z } ) - \mathbb { E } [ F ( \mathbf { z } ) ] \geq \epsilon \} \leq e ^ { - \frac { 2\epsilon ^ { 2} } { \Sigma _ { k = 1} ^ { 2} c _ { k } ^ { 2} } }
\end{eqnarray}
\end{lemma}	

Next, we will need to use the Rademacher average and its contraction property
\citep{bartlett2002rademacher} and \citep{meir2003generalization}.

\begin{definition}
Let $\mu$ be a a probability measure on $X$ and $F$ be a class of uniformly bounded functions. For every integer $m$, the \textbf{Rademacher average} over a set of functions on $F$ is denoted as:
\begin{eqnarray}\notag
R _ { m } ( F ) : = \mathbb { E } _ { \mu } \mathbb { E } _ { \epsilon } \left\{ \frac { 1} { m } \operatorname{sup} _ { f \in F } | \sum _ { i = 1} ^ { m } \sigma _ { i } f \left( z _ { i } \right) | \right\},
\end{eqnarray}
where $\left\{ z_ { i } \right\} _ { i = 1} ^ { m }$ are independent random variables distributed according to $\mu$ and $\left\{ \sigma _ { i } \right\} _ { i = 1} ^ { m }$are independent Rademacher random variables, i.e., $\operatorname{Pr}  \left( \sigma _ { i } = + 1\right) = \operatorname{Pr} \left( \sigma _ { i } = - 1\right) = 1/ 2$.
\end{definition}

\begin{lemma}
(\textbf{Rademacher contraction property}) Let $F$ be a class of uniformly bounded real-valued functions on $( X ,\mu )$
and $m \in \mathbb { N }$. If for each $i \in \{ 1,\dots ,m \}$, $\Psi _ { i } : \mathbb { R } \rightarrow \mathbb { R }$ is a function with a Lipschitz constant
$c_{i}$, then for any $\left\{ x _ { i } \right\} _ { i = 1} ^ { m }$, 
\begin{eqnarray}\notag
\mathbb { E } _ { \epsilon } \left( \operatorname{sup} _ { f \in F } | \sum _ { i = 1} ^ { m } \epsilon _ { i } \Psi _ { i } \left( f \left( x _ { i } \right) \right) | \right) \leq 2\mathbb { E } _ { \epsilon } \left( \operatorname{sup} _ { f \in F } | \sum _ { i = 1} ^ { m } c _ { i } \epsilon _ { i } f \left( x _ { i } \right) | \right).
\end{eqnarray}
\end{lemma}	

Using the standard techniques involving Rademacher averages in \citep{bartlett2002rademacher} and \citep{fan2017learning}, we can obtain the following results Lemma \ref{sample error 1} and \ref{lemma21}, which will be used to estimate the upper bound of the first part of sample error $\mathcal { S } ( m ,C )$ in Theorem \ref{theorem22}.

\begin{lemma}\label{sample error 1}
For any $\epsilon > 0$, there holds
\begin{eqnarray}\notag
\operatorname{Pr} \left\{ \operatorname{sup} _ { \| f \|_K \leq R } \left[\mathcal { E } ( f ) - \mathcal { E } _ {\mathbf { z } } ( f ) \right] \geq \epsilon + \frac { 2\kappa R\beta} { \sqrt { m } } \right\} \leq e ^ { - \frac { m \epsilon  ^ { 2} } { ( 1+ \kappa R ) ^ { 2}{\beta}^{2} } },
\end{eqnarray}
where $\beta=\frac{c^{-}\hat{c}}{c^{+}}$.
\end{lemma}	

\begin{proof}
Firstly, we rewrite the piece-wise loss function 
\begin{equation*}
V(y,f(x))=\left\{
\begin{matrix}
\frac{c^{-}\hat{c}}{c^{+}}(1-yf(x))_{+}, & \text{if~} x\in \mathcal{A_{+}} \text{~and~} y=1 \\
\frac{c^{-}}{c^{+}}(1-yf(x))_{+}, & \text{otherwise}
\end{matrix}
\right.
\end{equation*}
as 
\begin{equation}\label{loss1}
V(y,f(x))=\frac{c^{-}}{c^{+}}\left\{1+ \frac{[1-\chi(g^+(x))](y+1)(\hat{c}-1)}{4} \right\}(1-yf(x))_{+}
\end{equation}
where 
\begin{equation*}
\chi(g^+(x))=\left\{
\begin{matrix}
-1 & \text{if~} g^+(x)\leq 0\\
1, & \text{otherwise}.
\end{matrix}
\right.
\end{equation*}
Let 
\begin{equation*}
F (\mathbf{z} ) = \operatorname{sup} _ { \| f \| _ { K } \leq R } \left[ \mathcal { E } ( f ) - \mathcal { E } _ {\mathbf { z } } ( f ) \right],
\end{equation*}
where $\mathbf{ z }=\left\{z _ {1} ,\cdots ,z _ { k },\cdots, z _{m}\right\}$ is a training set with $m$ samples.

In order to present the techniques to prove this theorem clearly, we divided the proof into $3$ parts.

$(1)$ We will firstly show that the loss function Eq. (\ref{loss1}) has bounded differences.
\begin{align}
&\left.\begin{array} { c } { \max } \\ { z _ { 1} ,\cdots ,z _ { k } ,z _ { k } ^ { \prime } \dots ,z _ { m } } \end{array} \right. \left| F \left( z _ { 1} ,\cdots ,z _ { k - 1} ,z _ { k } ,z _ { k + 1} ,\cdots ,z _ { m } \right) - F \left( z _ { 1} ,\cdots ,z _ { k - 1} ,z _ { k } ^ { \prime } ,z _ { k + 1} ,\cdots ,z _ { m } \right) \right|\notag\\
&\leq\left.\begin{array} { c } { \max } \\ { z _ { 1} ,\cdots ,z _ { k } ,z _ { k } ^ { \prime } \dots ,z _ { m } } \end{array} \right. \left| \operatorname{sup} _ { \| f \| _ { K } \leq R } \left|\mathcal { E }_{\mathbf{z}} ( f ) - \mathcal { E } _ {\mathbf { z }^{\prime} } ( f )\right| \right|, \notag\\
\label{41}
&\leq\left.\begin{array} { c } { \max } \\ { z _ { 1} ,\cdots ,z _ { k } ,z _ { k } ^ { \prime } \dots ,z _ { m } } \end{array} \right. \left| \operatorname{sup} _ { \| f \| _ { K } \leq R }\frac { 1} { m }\left[V(y_{k},f(x_{k}))+V(y^{\prime}_{k},f(x^{\prime}_{k}))\right] \right|
\end{align}
where $\mathbf{ z }^{\prime}=\left\{z  _ {1} ,\cdots ,z ^{\prime}_ { k },\cdots, z _{m}\right\}$.

Because
\begin{eqnarray}\notag
V(y,f(x))\leq\frac{c^{-}\hat{c}}{c^{+}}(1+| f(x)| )=\frac{c^{-}\hat{c}}{c^{+}}(1+| \left\langle K _ { x } ,f \right\rangle _ { K } |)\leq\frac{c^{-}\hat{c}}{c^{+}}(1+\| f \| \left\langle K _ { x } ,K _ { x } \right\rangle _ { K } ^ { \frac { 1} { 2} })
\end{eqnarray}
and 
\begin{eqnarray}\notag
\frac{c^{-}\hat{c}}{c^{+}}(1+\| f \| \left\langle K _ { x } ,K _ { x } \right\rangle _ { K } ^ { \frac { 1} { 2} })=\frac{c^{-}\hat{c}}{c^{+}}(1+ \| f \| _ { K } \sqrt { K ( x ,x ) })\leq\frac{c^{-}\hat{c}}{c^{+}}(1+\kappa R).
\end{eqnarray}
Thus Eq. (\ref{41}) can be bounded by $\frac{2c^{-}\hat{c}}{mc^{+}}(1+\kappa R)$,
which means the bounded differences are 
\begin{equation}\notag
c_{k}= \frac{2c^{-}\hat{c}}{mc^{+}}(1+\kappa R),
\end{equation}
for any $1\leq k\leq m$.

(2) By the McDiarmid inequality, we have 
\begin{equation}
\operatorname{Pr} \left\{\operatorname{sup} _ {\| f \|_K \leq R } \left[ \mathcal { E } ( f ) - \mathcal { E } _ {\mathbf { z } } ( f ) \right] \geq \mathbb { E } _ { \mathbf { z } } \operatorname{sup} _ { \| f \|_K \leq R } \left[ \mathcal { E } ( f ) - \mathcal { E } _ {\mathbf { z } } ( f ) \right] + \epsilon \right\}\leq\exp \left\{ - \frac { 2m \epsilon ^ { 2} } { \beta^{2}( 1+ \kappa R ) ^ { 2} } \right\},\notag
\end{equation}
where $\beta=\frac{c^{-}\hat{c}}{c^{+}}$.
Let $\widetilde{\mathbf{ z }}=\left\{\widetilde{z}_ {1} ,\cdots ,\widetilde{z}_ { k },\cdots, \widetilde{z} _{m}\right\}$ be i.i.d. copies of $\mathbf{z}$, then
\begin{eqnarray}\notag
\mathbb { E } _ { \mathbf { z } } \operatorname{sup} _ { \| f \_K| \leq R } \left[ \mathcal { E } ( f ) - \mathcal { E } _ {\mathbf { z } } ( f ) \right] &=& \mathbb { E } _ { \mathbf { z } } \mathbb { E } _ {\widetilde{ \mathbf { z } }} ~\operatorname{sup} _ { \| f \|_K\leq R } \left[ \mathcal { E }_{\widetilde{ \mathbf { z } }} ( f ) - \mathcal { E } _ {\mathbf { z } } ( f )~|~\widetilde{ \mathbf { z } }~ \right]
\\&\notag\leq&\mathbb { E } _ { \mathbf { z } }\mathbb { E } _ {\widetilde{ \mathbf { z } }}~\operatorname{sup} _ { \| f \|_K\leq R } \left[ \mathcal { E }_{\widetilde{ \mathbf { z } }} ( f ) - \mathcal { E } _ {\mathbf { z } } ( f ) \right].
\end{eqnarray}
By standard symmetrization techniques\citep{bartlett2002rademacher}, For any Rademacher variables $\left\{ \sigma _ { i } : i = 1,\dots ,n \right\}$, there holds
\begin{eqnarray}\notag
&&\mathbb { E } _ { \mathbf { z } }\mathbb { E } _ {\widetilde{ \mathbf { z } }}~\operatorname{sup} _ { \| f \|_K\leq R } \left[ \mathcal { E }_{\widetilde{ \mathbf { z } }} ( f ) - \mathcal { E } _ {\mathbf { z } } ( f ) \right]\\\notag
 &=& \mathbb { E } _ { \mathbf { z } }\mathbb { E } _ {\widetilde{ \mathbf { z } }}\mathbb { E } _ { \sigma }~\operatorname{sup} _ { \| f \|_K\leq R }\left[\frac { 1} { m } \sum _ { i = 1} ^ { m }\sigma _ { i }\left\{V \left( \widetilde{y _ { i }} , f \left( \widetilde{x _ { i } }\right) \right)- V \left(y _ { i } , f \left(x _ { i }\right) \right)\right\} \right]\\\notag
&\leq&\mathbb { E } _ { \mathbf { z } }\mathbb { E } _ {\widetilde{ \mathbf { z } }}\mathbb { E } _ { \sigma }~\operatorname{sup} _ { \| f \|_K\leq R }\left[\frac { 1} { m } \sum _ { i = 1} ^ { m }\sigma _ { i }\left\{|V \left( \widetilde{y _ { i }} , f \left( \widetilde{x _ { i } }\right) \right)|+ |V \left(y _ { i } , f \left(x _ { i }\right) \right)|\right\} \right]\\\notag
&=&2\mathbb { E } _ { \mathbf { z } }\mathbb { E } _ {\sigma}~\operatorname{sup} _ { \| f \|_K\leq R }\frac { 1} { m } |\sum _ { i = 1} ^ { m }\sigma _ { i }V \left(y _ { i } , f \left(x _ { i }\right) \right)|
\end{eqnarray}
(3) Using the contraction property of Rademacher averages;

If we let the function:
\begin{eqnarray}\notag
\Phi _ { i } ( t ) =V \left(y _ { i } , t \right),
\end{eqnarray}
because for any $t_{1}, t_{2}\in \mathbb { R } $, there is 
\begin{eqnarray}\notag
| \Phi _ { i } ( t_{1})-\Phi _ { i } ( t_{2}) |\leq \frac{c^{-}\hat{c}}{c^{+}}|t_{1}-t_{2}|,
\end{eqnarray}
which means $\Phi _ { i } ( t )$ has the Lipschitz constant $\frac{c^{-}\hat{c}}{c^{+}}$. By the contraction of Rademacher averages, 
\begin{eqnarray}
\mathbb { E } _ {\sigma}~\operatorname{sup} _ { \| f \|_K\leq R }\frac { 1} { m } \left|\sum _ { i = 1} ^ { m }\sigma _ { i }V \left(y _ { i } , f \left(x _ { i }\right) \right)\right| 
&\leq& 2\mathbb { E } _ {\sigma}~\operatorname{sup} _ { \| f \|_K\leq R }\frac { 1} { m } \left|\sum _ { i = 1} ^ { m }\sigma _ { i }\frac{c^{-}\hat{c}}{c^{+}}f(x_{i})\right|\notag\\
&=&2 \mathbb { E } _ { \sigma } \operatorname{sup} _ { \| f \| _ { K } \leq R } \left| \left\langle \frac { 1} { m } \sum _ { i = 1} ^ { m } \sigma _ { i }\frac{c^{-}\hat{c}}{c^{+}} K _ { x _ { i } } ,f \right\rangle_{K} \right|\notag\\
&\leq&\mathbb { E } _ {\sigma}~\operatorname{sup} _ { \| f \|_K\leq R } \left\| \frac { 1} { m } \sum _ { i = 1} ^ { m} \sigma _ { i }\frac{c^{-}\hat{c}}{c^{+}} K _ { x _ { i } } \right\| _ { K } \left\| f \right\|_{K}\notag\\
&=&2R ~\mathbb { E } _ { \sigma } \left[ \left\| \frac { 1} { m } \sum _ { i = 1} ^ { m } \sigma _ { i }\frac{c^{-}\hat{c}}{c^{+}} K _ { x _ { i } } \right\| _ { K } \right]\notag\\
&\leq&\left[ \mathbb { E } _ { \sigma } \left\| \frac { 1} { m } \sum _ { i = 1} ^ { m } \sigma _ { i }\frac{c^{-}\hat{c}}{c^{+}}K _ { x _ { i } } \right\| _ { K } ^ { 2} \right] ^ { \frac { 1} { 2} }\notag\\
&\leq& 2R~\frac{c^{-}\hat{c}}{c^{+}}\left[\mathbb { E } _ { \sigma }\frac { 1} { m^{2} } \sum _ { i = 1} ^ { m } \sigma _ { i }^{2}K \left( x _ { i } ,x _ { i } \right)\right]^ { \frac { 1} { 2}}\notag\\
&\leq&\frac { 2\frac{c^{-}\hat{c}}{c^{+}}\kappa R } { \sqrt { m } }.\notag
\end{eqnarray}
Putting all the above estimations together yields the desired results. 
\end{proof}	

\begin{lemma}\label{lemma21}
For any $C>0$, $m \in \mathbb { N }$, $\mathbf { z } \in Z ^ { m }$, there holds
\begin{equation}
\| f _ { \mathbf { z } } \| _ { K }\leq \sqrt { 2 C\tilde{M} }
\end{equation}
where $\tilde{M}=\frac{c^{-}}{c^+}\frac{1}{m}(\hat{c}m_1+m_2)$.
\end{lemma}	

\begin{proof}
By the Definition of \ref{definition14} and choose $f=0+0$, we see that 
\begin{eqnarray}\notag
\mathcal { E } _ { \mathbf { z } } \left( f _ { \mathbf { z } } \right) + \frac { 1 } { 2 C } \| f _ { \mathbf { z } } \| _ { K } ^ { 2 } \leq \frac { 1 } { m } \sum _ { i = 1 } ^ { m } V \left( y _ { i } , 0 \right) + 0=\frac{c^{-}}{c^+}\frac{1}{m}(\hat{c}m_1+m_2).
\end{eqnarray}
this gives 
\begin{equation}\notag
\| f _ { \mathbf { z } } \| _ { K }\leq \sqrt { 2 C\tilde{M} },
\end{equation}
by choosing $\tilde{M}=\frac{c^{-}}{c^+}\frac{1}{m}(\hat{c}m_1+m_2)$.
\end{proof}	

Combining Lemma \ref{sample error 1} with  \ref{lemma21} and the choice $R=\sqrt { 2 C\tilde{M} }$, we can obtain the following result directly.

\begin{theorem}\label{theorem22}
For any $\epsilon > 0$, there holds
\begin{eqnarray}\notag
\operatorname{Pr} \left\{\left[\mathcal { E } ( f_ { \mathbf { z } } ) - \mathcal { E } _ {\mathbf { z } } ( f_ { \mathbf { z } } ) \right] \geq \epsilon + \frac { 2\kappa R\beta} { \sqrt { m } } \right\} \leq e ^ { - \frac { m \epsilon  ^ { 2} } { ( 1+ \kappa R ) ^ { 2}{\beta}^{2} } },
\end{eqnarray}
where $\beta=\frac{c^{-}\hat{c}}{c^{+}}$ and $R=\sqrt { 2 C\tilde{M} }$.	
\end{theorem}	

Theorem \ref{theorem22} indicates the upper bound for the first term of the sample error $\mathcal { S } ( m ,C )$, To derive the upper bound of the second term of $\mathcal { S } ( m ,C )$ in Theorem \ref{sample error 2}, we will need the Hoeffding’s inequality stated as follows.

\begin{lemma}\label{Hoeffding}
(\textbf{Hoeffding Inequality}) Let $\xi$ be a random variable and for any  $m$ values for the random variable $\xi$, $1 \leq i \leq m$, there exist $a_{i}, b_{i}$, such that $a _ { i } \leq \xi \leq b _ { i }$. Then, for any $\epsilon > 0$, there holds
\begin{eqnarray}\notag
\operatorname{Pr} \left\{ \frac { 1} { m } \sum _ { i = 1} ^ { m } \xi _ { i } - E \xi \geq \epsilon \right\} \leq \exp \left\{ - \frac { m \epsilon ^ { 2} } { 2M ^ { 2} } \right\},
\end{eqnarray}
where $M$ is a constant.
\end{lemma}

\begin{theorem}\label{sample error 2}
Let $f_{K,C}$ be defined in Definition \ref{definition12}, then there holds
\begin{eqnarray}\notag
\operatorname{Pr} \left\{\mathcal { E } _ { \mathbf { z } } \left( f _ { K ,C } \right) - \mathcal { E } \left( f _ { K ,C } \right) \geq \epsilon \right\} \leq \exp \left\{ - \frac { m \epsilon ^ { 2} } { 2\overline{M} ^ { 2} } \right\},
\end{eqnarray}
where $\overline{M}$ is a constant related to $M$. 
\end{theorem}

\begin{proof}
By the Definition \ref{definition12} and choose $f=0+0$, there holds
\begin{eqnarray}\notag
\mathcal { E } \left( f _ { K , C } \right) + \frac { 1 } { 2 C } \| f _ { K , C }  \| _ { K } ^ { 2 } \leq \int _ { Z } V ( y , 0 ) d \rho + 0\leq  \frac{c^{-}\hat{c}}{c^{+}},
\end{eqnarray}	
this gives $\| f _ { K , C }  \| _ { K } \leq \sqrt { 2\beta C }$, where $\beta=\frac{c^{-}\hat{c}}{c^{+}}$. Because 
$\| f _ { K , C }  \| _ { \infty } \leq \kappa \| f _ { K , C }  \| _ { K }$, where $\kappa = \operatorname{sup} _ { x \in X } \sqrt { K ( x ,x ) }$. Thus, there holds $\| f _ { K , C }  \| _ { \infty } \leq\sqrt { 2\beta C }$. By substituting the random variable $\xi$ in Theorem \ref{sample error 2} for $f _ { K , C }$, there obtains the result.
\end{proof}

Theorem \ref{theorem22} together with Theorem \ref{sample error 2} provides an upper bound of sample error in our learning scheme. If we assume the regularization error satisfies $\lim _ { C \rightarrow \infty } \mathcal { D } ( C ) = 0$, then the convergence property of our learning framework can be obtained.

\begin{corollary}\label{corollary}
Assume $\lim _ { C \rightarrow \infty } \mathcal { D } ( C ) = 0$. Choose the parameter $C = C _ { m }$ to satisfy 
$\lim _ { m \rightarrow \infty }C _ { m } =\infty$ and $\lim _ { m \rightarrow \infty } \frac{C_m}{m}=0$, then
\begin{equation}\notag
\lim _ { m \rightarrow \infty } \operatorname { Pr } \left\{\mathcal { U } \left( f _ { q } \right)-\mathcal { U } \left( \operatorname { sgn } \left( f _ { \mathbf { z } } \right) \right) > \varepsilon \right\} = 0.
\end{equation}
\end{corollary}	

\begin{proof}
This assertion follows from Theorem \ref{theorem22}	and Theorem \ref{sample error 2}.
\end{proof}	

For completeness, we now derive an upper bound for regularization error.

\begin{theorem}\label{regularization error}
For every $C > 0$, there holds
\begin{eqnarray}\notag
\mathcal { D } ( C ) \leq \frac{c^{-}\hat{c}}{c^{+}} \operatorname{inf} _ { f \in {\mathcal { H }} _ { K } } \left\{ \| f - f _ { q } \| _ { L _ { \rho_{X} } ^ { 1} } + \frac { 1} { 2C } \| f \| _ { K } ^ { 2} \right\}.
\end{eqnarray}
\end{theorem}
\begin{proof}
By the definition $\mathcal { D } ( C ) : = \operatorname{inf} _ { f \in { \mathcal { H } } _ { K } } \left\{ \mathcal { E } ( f ) - \mathcal { E } \left( f _ { q } \right) + \frac { 1} { 2C } \| f \| _ { K } ^ { 2} \right\}$	
and 
\begin{eqnarray}\notag
\mathcal { E } ( f )=\mathbb{E}\left\{
V(\mathcal {Y},f(\mathcal {X}))
\right\}=\mathbb{E}\left\{\mathbb{E}\left\{
V(\mathcal {Y},f(\mathcal {X}))~|~\mathcal {X}=x
\right\}\right\}.
\end{eqnarray}
Notice that
\begin{equation}\notag
V(y,f(x))=\frac{c^{-}}{c^{+}}\left\{1+ \frac{[1-\chi(g^+(x))](y+1)(\hat{c}-1)}{4} \right\}(1-yf(x))_{+}.
\end{equation}
where 
\begin{equation*}
\chi(g^+(x))=\left\{
\begin{matrix}
-1 & \text{if~} g^+(x)\leq 0\\
1, & \text{otherwise}.
\end{matrix}
\right.
\end{equation*}
If we denote 
\begin{equation*}
d(x,y)=\frac{c^{-}}{c^{+}}\left\{1+ \frac{[1-\chi(g^+(x))](y+1)(\hat{c}-1)}{4} \right\}.
\end{equation*}
Then
\begin{eqnarray}\notag
&&\mathbb{E}\left\{
V(\mathcal {Y},f(\mathcal {X}))~|~\mathcal {X}=x
\right\}-\mathbb{E}\left\{
V(\mathcal {Y},f_{q}(\mathcal {X}))~|~\mathcal {X}=x
\right\}  \\\notag
&=&\int _ { Y } V( y,f(x) )  - V\left( y,f_{q}(x) \right)d\rho ( y | x )\\ \notag
&=&\int _ { Y } d(x,y)\left[( 1- y f ( x ) ) _ { + } - \left( 1- yf_{q}( x ) \right) _ { + } \right]d \rho ( \mathcal {Y} | x )
\\\notag
&\leq& \frac{c^{-}\hat{c}}{c^{+}} \int _ { Y } |( 1- y f ( x ) ) _ { + } - \left( 1- yf_{q}( x ) \right) _ { + }| d \rho ( \mathcal {Y} | x )
\end{eqnarray} 
Since function $\left( 1- yf_{q}( x ) \right) _ { + }$ is Lipschitz:
\begin{equation*}
|\left( 1- yf( x ) \right) _ { + }-\left( 1- yf_{q}( x ) \right) _ { + }|\leq |f(x)-f_{q}(x)|
\end{equation*}
Finally,  we can get the upper bound of the regularization error.
\end{proof}	

From the result on Theorem \ref{regularization error}, we notice that the regularization error can be estimated by the approximation in a weighted $L ^ { 1}$ space. Thus, we have the claim: for a distribution $\rho$ such that $f_{q}$ lies in the closure of $\mathcal { H } _ { K }$ in $ L _ { \rho_{X} } ^ { 1}$, the regularization error tends to $0$, as $C \rightarrow \infty$. This together with Corollary \ref{corollary} gives the convergence (in the sense of utility function) property of the proposed knowledge-based SVM (\ref{modified svm2}) for the distribution. 


In order to provide a more quantitative convergence result of the proposed knowledge-based classifier, we give a Corollary below which is similar to the result in   \citep{wu2006analysis}. 
\begin{corollary}\label{corollary27}
Let $X = [ 0,1 ] ^ { n } , \sigma > 0,0 < s$, and assume $K$  to be a Guassian kernel, if $\frac { d \rho _ { X } ( x ) } { d x } \leq C _ { 0 },$
for all most every $x \in X$, $f _ { q }$ is the restriction of some function $\tilde { f } _ { q } \in H ^ { s } \left( \mathbb { R } ^ { n } \right)$ onto $X$, and choose $C$ such that $\lim _ { m \rightarrow \infty } \frac{C}{m}=0$, then with probability at least $1- \delta$, there holds 
\begin{equation}\notag
\mathcal { U } \left( f _ { q }\right)-\mathcal { U } \left( \operatorname { sgn } \left( f _ { \mathbf { z } } \right) \right)  \leq \mathcal { O } \left( \frac{1}{\sqrt{m}} \right) +\mathcal { O } \left( ( \log m ) ^ { - s } \right).
\end{equation}	
\end{corollary}	

\begin{proof}
First, we estimate the sample error;

For every $0<\delta<1$,  set 
\begin{equation}\notag
\exp\left\{ - \frac { m \epsilon  ^ { 2} } { ( 1+ \kappa R ) ^ { 2}{\beta}^{2} } \right\}
=\frac{\delta}{2},
\end{equation}
and 
\begin{equation}\notag
\exp \left\{ - \frac { m \epsilon ^ { 2} } { 2\overline{M} ^ { 2} } \right\}=\frac{\delta}{2}.
\end{equation}
By solving above two equations and according to Theorem \ref{theorem22} and \ref{sample error 2}, with confidence at least $1-\delta$, there holds
\begin{equation}\notag
\mathcal { S } ( m , C ) \leq \varepsilon ( \delta , m , C,\beta )+\frac { 2\kappa R\beta} { \sqrt { m } }=\mathcal{ O }\left( \frac{1}{\sqrt{m}} \right)+\frac { 2\kappa R\beta} { \sqrt { m } }
\end{equation}
where $R= \sqrt { 2 C\tilde{M} }$ and $\beta=\frac{c^{-}\hat{c}}{c^{+}}$.

If we choose $C$ such that $\lim _ { m \rightarrow \infty } \frac{C}{m}=0$,
thus, there follows 
\begin{equation}\notag
\mathcal { S } ( m , C ) \leq\mathcal{ O }\left( \frac{1}{\sqrt{m}} \right).
\end{equation}	
Second, we begin to estimate the generalization error;	
	
Denote 
\begin{equation}\notag
I _ { 1 } ( g , \bar{R} ) = \inf _ { f \in { \mathcal { H } } _ { K } , \| f \| _ { K } \leq \bar{R} } \left\{ \| g - f \| _ { L _ { \rho _ { X } } ^ { 1 } } \right\},
\end{equation} 
and
\begin{equation}\notag
I _ { 2 } ( g , \bar{R} ) = \inf _ { f \in { \mathcal { H } } _ { K } , \| f \| _ { K } \leq \bar{R} } \left\{ \| g - f \| _ { L _ { \rho _ { X } } ^ { 2 } } \right\},
\end{equation}	
then there hold
\begin{eqnarray}\notag
\mathcal { D } ( C ) &\leq& \beta \operatorname{inf} _ { f \in { H } _ { K } } \left\{ \| f - f _ { q } \| _ { L _ { \rho_{X} } ^ { 1} } + \frac { 1} { 2C } \| f \| _ { K } ^ { 2} \right\}\\\notag
&\leq& \beta \inf _ { \bar{R} > 0 } \left\{ I _ { 1 } \left( f _ { q } , \bar{R} \right) + \frac { \bar{R} ^ { 2 } } { 2 C } \right\}\\\notag
&\leq& \beta \inf _ { \bar{R} > 0 } \left\{ I _ { 2 } \left( f _ { q } , \bar{R} \right) + \frac { \bar{R} ^ { 2 } } { 2 C } \right\}
\end{eqnarray}	
The last inequation follows from Holder inequality, then according to the approximation error estimated in \citep{smale2003estimating} and \citep{wu2006analysis}, we have
\begin{eqnarray}\notag
I _ { 2 } \left( f _ { q } , \bar{R} \right) \leq C _ { 0 } C _ { s } ( \log \bar{R} ) ^ { - s / 4 } , \quad \forall \bar{R} > C _ { s } ,
\end{eqnarray}	
where $C_s$ is a constant depending on $s, \sigma, n$ and $\| \tilde { f } _ { q } \| _ { H ^ { s } } + \| \tilde { f } _ { q } \| _ { L ^ { 2 } }$. 

Choose $\bar{R}$ to be $\sqrt { C } ( \log C ) ^ { - s / 4 - 1 }$, then we obtain 
\begin{equation}\notag
\mathcal { D } ( C ) \leq \mathcal { O } \left( ( \log C ) ^ { - s} \right).
\end{equation}
Then choose $C=m^\gamma$ with $0<\gamma<1$, and
combing the estimates for the sample error and the generalization error, our statement holds.
\end{proof}	

Corollary \ref{corollary27} indicates that $\mathcal { U } \left( \operatorname { sgn } \left( f _ { \mathbf { z } } \right) \right)$ can arbitrarily close to $\mathcal { U } \left( f _ q \right)$, which high probability, as long as $m$ is sufficiently large. In other words, the designed knowledge-based SVM classifier $ f _ { \mathbf { z } } $ can (asymptotically) approximate the optimal classifier $f_q$, which implies that the built knowledge-based SVM (\ref{modified svm2}) has successively achieved our learning target.

 

\section{Discussions}

In this section, we try to demonstrate the significance of our knowledge incorporation learning scheme in three aspects. Firstly, we will illustrate the generality of our classifier designing method. Secondly, the comparisons with two knowledge incorporation models \citep{wu2004incorporating} and \citep{mangasarian2008nonlinear} are analyzed. Finally, we attempt to explain the (geometric) meaning of the proposed piece-wise loss function Eq. (\ref{piecewise loss function}). 

\subsection{The Generality of the Knowledge-based SVM}
 
It is worth to note that our knowledge incorporation scheme is very general. Actually, the concept of the proposed utility function can be generalized to some extent. For example, one can modify the formula of the utility function to achieve different learning goals if some other prior knowledge is available. 

Also, the designed knowledge-based classification model (\ref{modified svm2}) can be degenerated to the model (\ref{svm2}) presented in \citep{lin2002support} when limiting $\hat{c}=1$. Thus, our analysis framework directly includes the case study of model (\ref{svm2}). Specifically, our results strictly verified that the classifier of model (\ref{svm2}) can asymptotically converge to the classifier ${f_q}^\prime$ which minimizes the expected cost Eq. ({\ref{4}}) and the learning rate is $\mathcal { O } \left( ( \log m ) ^ { - s } \right)$.
\begin{eqnarray}\notag
{f_q}^\prime (x) = \left\{ \begin{array} { l l } { + 1 } & { \text { if } \frac { \Pr \left(Y=1~|~\mathcal { X }=x\right) } { 1 - \Pr \left(Y=1~|~\mathcal { X }=x\right)} > \frac { c ^ { + } } { c ^ { - } } } \\ { - 1 } & { \text { otherwise } } \end{array} \right.
\end{eqnarray}
Further, if we assume $\hat{c}=1$ and $c^+=c^-$, our knowledge-based SVM (\ref{modified svm2}) reduces to the original SVM, which again verifies that the classifier of original SVM asymptotically approaches the Bayes rule \citep{chen2004support}, \citep{wu2006analysis}. Also, our quantitative convergence result is coincident with the result provided in \citep{wu2006analysis}.

\subsection{Comparisons to Other Knowledge Incorporation Schemes}

 In this part, we want to compare our knowledge-based SVM (\ref{modified svm2}) with two knowledge incorporation models (\ref{svm3}) and (\ref{generalized svm4}) in deeper levels. 
 
 We first compare these models from a perspective of the prior knowledge.  The type of prior knowledge used in the model (\ref{svm3}) fundamentally describes the importance of every data samples to the decision plane which is sample-based. We can understand the role of this knowledge is to do sample selection. While, Eq. (\ref{pr3}) in model (\ref{generalized svm4}) reflects the geometric information about the decision plane which can be seen as the region-based prior knowledge. However, the proposed prior knowledge in this paper Eq. (\ref{prior1}) is in the middle of sample-based and region-based, which comes from the domain experts experience.
 
 Secondly, we compare these models by estimating their generalization errors. Actually, the loss function corresponds to the model (\ref{svm3}) is 
 \begin{equation}\notag
 \tilde{V}(y,f(x))=h(x)(1-yf(x))_+ 
 \end{equation}
where the function $h$ is the a monotonically increasing function as defined in model (\ref{generalized svm4}) and actually it depends on the sample points. Minimizing the generalization error $\mathbb{E}[\tilde{V}(\mathcal {Y},f(\mathcal {X}))]$ over all measurable functions will generate the Bayes rule $f_c$, by following the standard techniques in Theorem \ref{theorem 1}. 

 As for the model (\ref{generalized svm4}), if we assume the training data set in the model (\ref{generalized svm4}) is  $\textbf{z} = \left( z _ { i } \right) _ { i = 1} ^ { m+p }=\left( x _ { i } ,y _ { i } \right) _ { i = 1} ^ {m+p}$, which are i.i.d sampled according to a Borel probability measure $\bar{\rho}$ on the space $Z : = X \times Y$. Then, we can easily find out that the loss function with respect to the model (\ref{generalized svm4}) is 
 \begin{eqnarray}\notag
 \bar{V}(y,f(x))=\left\{
 \begin{matrix}
 (\eta(x)-yf(x))_{+}, & \text{if~} g(x)\leq0 \text{~and~} y=1 \\
 (1-yf(x))_{+}, & \text{otherwise}
 \end{matrix}
 \right.
 \end{eqnarray}
 where $\eta(x)=1- v g\left({x} \right)\geq 1$.
 
It is easy to check that the minimizer of $\mathbb{E}[\bar{V}(\mathcal {Y},f(\mathcal {X}))]$ is $f_c$, the Bayes rule, by following the standard techniques in Theorem \ref{theorem 1}. 

These two results are quite interesting which implies that the asymptotic performances of models (\ref{svm3}), (\ref{generalized svm4}) are both equivalent to the standard Bayes rule, even though the structure of models (\ref{svm3}), (\ref{generalized svm4}) are non-standard. However, the proposed knowledge-based SVM (\ref{modified svm2}) classifier asymptotically amounts to the classifier $f_q$ to optimize the utility function.

In summary, the essential differences between the proposed knowledge-based SVM and models (\ref{svm3}), (\ref{generalized svm4}) can be reflected on the different asymptotic performances of these models.

Notice that in our proposed knowledge-based SVM, we don't need to assume the convex property of the function $g^{+}({x})$. Also as indicated previously, users can adjust three parameters $\hat{c}, c^+, c^-$ flexibly according to different demands in real implementations, which indicates the advantages of our model.

\subsection{Comparisons to Other piece-wise loss functions}
Mathematically, our knowledge-based SVM learning framework falls into the category of 
learning with piece-wise loss function in batch settings. Concretely, the designed piece-wise loss Eq.  (\ref{piecewise loss function}) is segmented by both feature and label aspects. 
There are now many machine learning problems involving piece-wise loss functions which depends on a pair of sample points.
For example, metric learning \citep{davis2007information}, \citep{jin2009regularized}, \citep{weinberger2009distance}, \citep{ying2012distance}, \citep{bohne2014large} aims to learn a metric $D$ such that examples with the same label stay closer while pushing apart examples with distinct labels which produces a typical pairwise loss function: $\ell \left( f ,( x ,y ) ,\left( x ^ { \prime } ,y ^ { \prime } \right) \right)=\left( 1+ r \left( y ,y ^ { \prime } \right) D \left( x ,x ^ { \prime } \right) \right)_+$ where $r \left( y ,y ^ { \prime } \right)=1$ if $y = y ^ { \prime }$ and $-1$ otherwise. Notice that the loss function for metric learning is segmented by labels which reflects the user's comprehension of similarity between the sample points. 

Thus, motivated from the above example, we can understand the piece-wise loss function Eq.  (\ref{piecewise loss function}) in our knowledge incorporation scheme as enlarging the error metric when samples in the region $g^+(x)\leq0$.
It is worth to point out that the proposed piece-wise loss function implies that we can combine human knowledge with black-box modeling scheme from the perspective of the error measures, which further plays a key role on theoretically interpreting the effect of human experience in black-box modeling.

\section{Conclusion}

In this paper, we propose an innovative framework to design the classifier that meets the desired learning target for optimizing the utility function. Precisely, as the data size grows, we prove that the produced classifier asymptotically converges to the optimal classifier, an extended version of the Bayes rule, which maximizes the utility function. Therefore, we provide a meaningful theoretical interpretation for modeling with the knowledge incorporated. Further, our knowledge incorporation scheme offers the domain experts a way to interact with the machine learning system, so that they can understand it better.

The main trick in this paper is to integrate prior knowledge into the learning target from which the knowledge-based SVM classifier is derived. Therefore, our knowledge incorporation method is new.
More importantly, our analysis indicates that the key to understanding the machine learning methods may not be the models themselves, but the corresponding expected performance measures, i.e., the utility function, which determines the structures of machine learning models.

It is necessary to notice that choosing a classification learning framework to optimize the utility function can be transformed into minimizing empirical error for the piece-wise loss function Eq. (\ref{piecewise loss function}), which implies that we can use other powerful algorithms to achieve this goal, such as online learning, etc.  We can also extend our analysis framework associated with the proposed loss function Eq. (\ref{piecewise loss function}) to online settings, i.e., we can perform gradient descent with respect to the instantaneous loss $V(y_t,f(x_t))$. And this will leave for our future work.

\acks{We would like to acknowledge support for this project
from the National Natural Science Foundation of China under Grant No. 11671418 and 61611130124.}


\newpage








\vskip 0.2in

\end{document}